\documentclass[letterpaper]{article} 
\usepackage{aaai2026}  
\usepackage{times}  
\usepackage{helvet}  
\usepackage{courier}  
\usepackage[hyphens]{url}  
\usepackage{graphicx} 
\usepackage{natbib}  
\usepackage{caption} 
\usepackage{algorithm}

\usepackage{newfloat}
\usepackage{listings}
\DeclareCaptionStyle{ruled}{labelfont=normalfont,labelsep=colon,strut=off} 
\floatstyle{ruled}
\newfloat{listing}{tb}{lst}{}
\floatname{listing}{Listing}
\usepackage{booktabs}       
\usepackage{amsfonts}       
\usepackage{nicefrac}       
\usepackage{microtype}      
\usepackage{xcolor}         
\usepackage{graphicx}
\usepackage{amsthm}
\usepackage{caption}
\usepackage{multirow}
\usepackage{adjustbox}
\usepackage{amssymb}
\usepackage{mathtools}
\usepackage{algorithm}       
\usepackage{algorithmicx}    
\usepackage{algpseudocode}   
\usepackage{amsmath}         
\theoremstyle{plain}
\usepackage{adjustbox}  
\usepackage{amssymb}
\usepackage{nicematrix}
\usepackage{booktabs}  
\newtheorem{theorem}{Theorem}[section]
\newtheorem{proposition}[theorem]{Proposition}

\theoremstyle{definition}

\theoremstyle{remark}

\usepackage{dblfloatfix} 

\usepackage{newfloat}
\usepackage{listings}
\DeclareCaptionStyle{ruled}{labelfont=normalfont,labelsep=colon,strut=off} 
\floatstyle{ruled}
\newfloat{listing}{tb}{lst}{}
\floatname{listing}{Listing}
\title{Bridging the Last Mile of Prediction: Enhancing Time Series Forecasting with Conditional Guided Flow Matching}
\author{%
  Huibo Xu\textsuperscript{1}\qquad 
  Runlong Yu\textsuperscript{2} \qquad 
  Likang Wu\textsuperscript{3} \qquad 
  Xianquan Wang\textsuperscript{1} \qquad 
  Qi Liu\textsuperscript{1}\thanks{Corresponding author: \texttt{qiliuql@ustc.edu.cn}}  \\
  \small{\textsuperscript{1}University of Science and Technology of China} \qquad
  \small{\textsuperscript{2}University of Pittsburgh} \qquad
  \small{\textsuperscript{3}Tianjin University} \\[3pt]
}

\usepackage{bibentry}

\begin{document}

\maketitle

\begin{abstract}
Existing generative models for time series forecasting often transform simple priors (typically Gaussian) into complex data distributions. However, their sampling initialization, independent of historical data, hinders the capture of temporal dependencies, limiting predictive accuracy. They also treat residuals merely as optimization targets, ignoring that residuals often exhibit meaningful patterns like systematic biases or nontrivial distributional structures. To address these, we propose Conditional Guided Flow Matching (CGFM), a novel model-agnostic framework that extends flow matching by integrating outputs from an auxiliary predictive model. This enables learning from the probabilistic structure of prediction residuals, leveraging the auxiliary model’s prediction distribution as a source to reduce learning difficulty and refine forecasts. CGFM incorporates historical data as both conditions and guidance, uses two-sided conditional paths (with source and target conditioned on the same history), and employs affine paths to expand the path space, avoiding path crossing without complex mechanisms, preserving temporal consistency, and strengthening distribution alignment. Experiments across datasets and baselines show CGFM consistently outperforms state-of-the-art models, advancing forecasting.

\end{abstract}


\section{Introduction}

Time series forecasting, a fundamental task in time series analysis, is widely used and has a considerable impact in various domains such as finance, healthcare, and energy~\cite{lim2021time}. Traditionally viewed as the process of predicting future values from historical observations, it can also be framed as a conditional generation problem, where future data are generated conditioned on history data~\cite{shen2023non}.  Naturally, Generative models have thus emerged as powerful tools for this task~\cite{li2022generative,yoon2019time}, with diffusion models achieving success for their predictive distribution modeling capabilities and conditional generation mechanisms, achieving notable success in time series forecasting~\cite{kollovieh2024predict}. 


Currently, most generative models for time series forecasting transform a simple prior, typically Gaussian, into a complex data distribution, often under rigid generation constraints. However, because sampling is initialized independently of historical data, these models struggle to capture temporal dependencies, which in turn limits predictive accuracy. In addition, forecasting models typically treat residuals purely as optimization targets, minimizing them through loss functions without modeling them as structured, learnable signals. Yet in practice, residuals often exhibit meaningful patterns, such as systematic biases or nontrivial distributional structures. Modeling these residuals, rather than discarding them after optimization, represents a promising direction for enhancing forecasting performance.


\begin{figure}[t]
    \centering
    \includegraphics[width=0.85\linewidth]{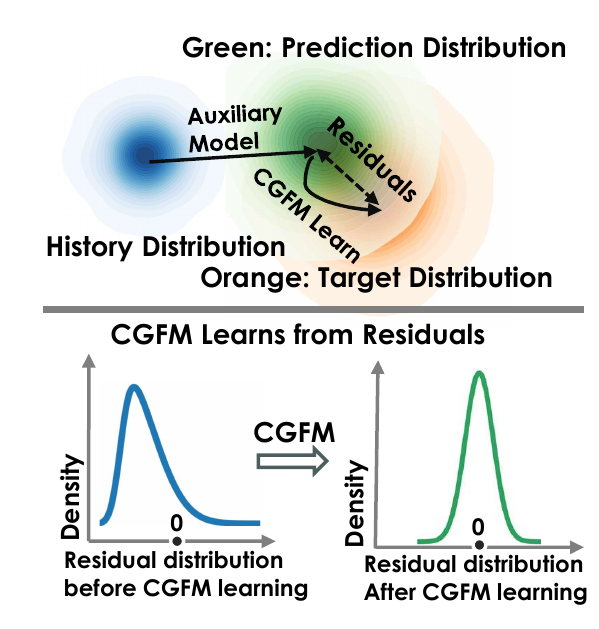}
    \caption{Top: Auxiliary models learn a prediction distribution that differs from the target, leading to residuals. CGFM learns the probabilistic structure of these residuals. Bottom: CGFM reduces the residuals to be centered and concentrated around zero (mean tending to zero with significantly reduced variance).}
    \vspace{-0.4cm}
    \label{fig:your-label}
\end{figure}

To address these challenges, we turn to flow matching~\cite{esser2024scalingrectifiedflowtransformers,kerrigan2023functionalflowmatching,lipman2023flowmatchinggenerativemodeling}, a state-of-the-art generative modeling framework that offers greater flexibility in both the choice of initial distributions and the design of sampling trajectories, while also producing higher-quality results compared to diffusion models. Leveraging these advantages, we propose \textbf{Conditional Guided Flow Matching (CGFM)}, a novel forecasting framework designed to enhance predictive accuracy by learning the probabilistic structure of residuals between model predictions and true future values, thereby refining predictions toward the target distribution.


First, we utilize the distribution of an auxiliary model’s predictions as the source distribution, rather than constraining it to a simple noise distribution. On one hand, the prediction distribution of the auxiliary model preserves richer temporal dependencies compared to a Gaussian prior and serves as the closest accessible approximation to the target distribution, naturally reducing the learning difficulty when starting from a distribution closer to the target. On the other hand, as we will demonstrate later, this flow matching formulation is equivalent to learning the probabilistic characteristics of the residuals, going beyond merely minimizing loss functions. By exploiting the valuable information encoded in the prediction residuals, our approach further refines the results and achieves improved forecasting accuracy.

Second, the essence of forecasting lies in capturing the relationship between historical data and future target values. We innovatively incorporate historical data into both probability path construction and velocity field learning. Key innovations include a two-sided conditional probability path, in which both the source and target distributions are conditioned on the same historical data, as a generalization of linear ones. This two-sided design, together with our choice of source distribution, model structure, and the flexibility of affine paths, theoretically guarantees the avoidance of path crossing, a critical challenge in flow-based modeling, without relying on complex restrictive mechanisms such as OT-CFM’s~\cite{tong2024improvinggeneralizingflowbasedgenerative}  optimal transport plans or Rectified Flow’s~\cite{liu2022flowstraightfastlearning} specialized training schemes. Affine paths expand the space of probability paths, enabling more flexible adaptation to varying initial distributions. By preventing path crossing, the framework reduces prediction ambiguity and information loss during sampling, which contributes to improved forecasting performance. Historical data guides the velocity field to effectively capture temporal dependencies, further strengthening the alignment between source and target distributions. Additionally, reparameterizing the prediction target to directly optimize toward the target data further enhances accuracy.

Our main contributions are as follows:
\setlength{\leftmargini}{20pt}
\begin{itemize}

\item We propose a novel model-agnostic time series forecasting framework, extending flow matching by integrating the outputs of an auxiliary model. This enables learning from prediction residuals, leveraging the auxiliary model's prediction distribution as a source to reduce learning difficulty and refine forecasts 

\item We innovatively integrate historical data as two-sided conditions into both probability path construction and velocity field learning: by designing two-sided conditional probability paths and combining general affine paths, we naturally avoid path crossing without relying on complex restrictive mechanisms. This design inherently fits the temporal dependencies of time series data, preserves temporal consistency, and strengthens distribution alignment, thereby boosting accuracy.

\item We conduct extensive experiments, demonstrating that CGFM consistently improves forecasting performance and outperforms state-of-the-art models, validating its effectiveness and generality.

\end{itemize}

\section{Preliminaries}

\subsection{Flow Matching}

Given a sample \( X_0 \) drawn from a source distribution \( p \) such that \( X_0 \sim p \), in \( d \)-dimensional Euclidean space where \( X_0 = (x_0^1, \ldots, x_0^d) \in \mathbb{R}^d \), and a target sample \( X_1 = (x_1^1, \ldots, x_1^d) \) \( X_1 \sim q \). Flow Matching (FM) constructs a probability path \((p_t)_{0 \leq t \leq 1}\) from the known distribution \(p_0 = p\) to the target distribution \(p_1 = q\), where \(p_t\) is a distribution over \(\mathbb{R}^d\). Specifically, Flow Matching employs a straightforward regression objective to train the velocity field neural network, which describes the instantaneous velocities of samples. 
The relationship between the velocity field and the flow is defined as:
\begin{equation}
\frac{d}{dt} \psi_t(x) = u_t(\psi_t(x))
\label{eq:conditional_flow},
\end{equation} where \( \psi_t(x) \) represents the flow at time \( t \), and \( \psi_0(x) = x \). The velocity field \( u_t \) generates the probability path \( p_t \) if its flow \( \psi_t \) satisfies \( X_t := \psi_t(X_0) \sim p_t\) for \(X_0 \sim p_0. \) The goal of Flow Matching is to learn a vector field \( u_{\theta}(t) \) such that its flow \( \psi_t \) generates a probability path \( p_t \) with \( p_0 = p \) and \( p_1 = q \). The Flow Matching loss is defined as: \(\mathcal{L}_{\text{FM}}(\theta) = \mathbb{E}_{t, X_t} \left[ \left\| u_{t}( X_t) - u_t^\theta(X_t) \right\|^2 \right], \)
\begin{figure*}
    \centering
    \includegraphics[width=0.8\linewidth]{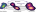}
\caption{Visualization of CGFM training. The ideal predictor maps the history distribution to the target distribution. However, in practice, the prediction distribution of the auxiliary model inevitably deviates from the target distribution. CGFM enhances prediction by learning the path between the prediction and target distributions. The black line between \(x_0\) and \(x_1\) represents the two-sided conditional guided probability path, while \(h\) generates \(x_0\) and \(x_1\) along the gray line, facilitating the learning of this path.}
\vspace{-1ex}
    \label{fig:enter-label}
\end{figure*}
\section{Related Work}
We present the most relevant related work here. Additional references are provided in the appendix.
\subsection{Flow Matching for Time Series}
Flow Matching (FM) has received increasing attention in time series modeling for its efficiency in constructing continuous probability paths via velocity fields. Early work such as CFM-TS~\cite{tamir2024conditional} explored FM for time series generation using Brownian bridges and Gaussian processes, offering improved stability over neural ODEs. TSFlow~\cite{kollovieh2024flowmatchinggaussianprocess} further addressed the large distributional gap between source and target by coupling GP-informed priors with optimal transport, though it only incorporated historical data during inference, limiting its temporal modeling capacity during training. FM-TS~\cite{hu2024fm} applied rectified flows in a basic form, demonstrating applicability but lacking structured conditioning. 
TFM~\cite{zhang2024trajectory} focused on clinical forecasting with memory-based trajectory coupling for irregular sampling and covariates but is limited by historical window constraints and general priors, underutilizing complex temporal dependencies and structured signals. CGFM's bilateral conditional design—with non-crossing paths, residual learning, and model independence—advances temporal consistency, efficiency, and flexibility for time series forecasting. Despite progress, existing methods suffer from \textit{weak historical data use} and over-reliance on generic priors. Our \textbf{CGFM} addresses this via an auxiliary model's predictive distribution (more target-closer than GP priors) to correct residuals, plus \textit{two-sided conditional paths} (historical-guided) for consistency and non-crossing; with affine interpolation, it captures residual dynamics and boosts accuracy.

\section{CGFM}

In time series prediction tasks, the goal is to leverage historical data to predict future data. Let \(H \in \mathbb{R}^{C \times L} \sim p_H\) denote the historical data, with samples represented by \(h\). Similarly, let \(F \in \mathbb{R}^{C \times F} \sim q\) represent the target future data, with samples denoted by \(f\). The probability path is defined as \(p_t\), where \(X_t \sim p_t\) represents the state of the data at time \(t\). The objective is to predict the future data \(F\), where in the flow matching framework, \(q\) serves as the target distribution. Specifically, at \(t=1\), \(X_1 \in \mathbb{R}^{C \times F} \sim q\), with samples denoted by \(x_1\). At \(t=0\), the source data is given by \(X_0 \in \mathbb{R}^{C \times F} \sim p\), with samples represented by \(x_0\).

\subsection{Choice of $X_0$}

To fully leverage the valuable information encoded in the predictions residuals, we harness the flexibility of flow matching by setting \(X_0\) as the output generated by a predictive auxiliary model, i.e., \(X_0 = X_{\text{aux}} = \Phi(H)\). In this case, the distribution \(p\) conditioned on \(h\) can be expressed as \(p(x_0 \mid h)\), representing the distribution of the auxiliary model's predictions. Specifically, we design a probability path that from the distribution of the auxiliary model's predictions to the target distribution. This allows flow matching to effectively learn from prediction residuals of the auxiliary model. 



\subsection{Analysis of the \(p(x_0\mid h)\)  Distribution}
Since time series datasets are typically stored with a limited number of significant digits and consist of a finite number of samples, even if the original time series is continuous, $H$ can only be considered to have an approximately continuous distribution in $\mathbb{R}^{C \times L}$. In general, most models are differentiable, and thus \(\Phi(H)\) can also be approximated as \( C(\mathbb{R}^{C \times F}) \). However, this is insufficient to ensure that \(p(x_0\mid h)\) is \( C^1 \) in \(\mathbb{R}^{C \times L}\).

\begin{proposition}[Noise Smoothing]
    Let \( X \in \mathbb{R}^{C \times F} \) be a time series with distribution \( P_{\text{ori}} \). Define the perturbed series \( R \) as: \( R = X + \sigma \epsilon,\) where \( \epsilon \sim \mathcal{N}(0, I) \) is additive Gaussian noise and \( \sigma > 0 \). The perturbed series \( R \) follows the distribution \( P_{\text{per}} \):
    \begin{equation}
        P_{\text{per}}(r) = \int P_X(r - \sigma \epsilon) P_\epsilon(\epsilon) \, d\epsilon,
    \end{equation}
which belongs to \( C^1(\mathbb{R}^d) \). Furthermore, \( P_{\text{per}} \) has a strictly positive density and possesses finite second moments.
\label{noise_smooth}
\end{proposition}
\vspace{-0.2cm}

Intuitively, \( P_{\text{per}} \) can be regarded as the convolution of the original distribution \( P_X(x) \) with a Gaussian distribution \( P_\epsilon(\epsilon) \). Since the Gaussian distribution is a \( C^\infty \) function, \( P_R(r) \) is thus not only \( C^1 \) but also \( C^\infty \). When \(\Phi(H) + \sigma \epsilon\) is applied, the convolution with Gaussian noise significantly enhances the smoothness of the distribution, eliminating sharp variations and discontinuities present in the original distribution. This process ensures that the source remains a valid density while promoting diversity, aligning the noise-smoothed source distribution \( p \) with Proposition \ref{marginalization}, which will be discussed later.

\subsection{Two-Sided Conditional Guided Prediction}

For flow matching where the initial distribution is set to the predictions of the auxiliary model, the target distribution \(q\) used during training is derived from the inherent correspondence between future data \(F\) and historical data \(H\) in the dataset, which can be formulated as \(q(x_1 \mid h)\).

The next step is to learn the probability path from \(p(x_0 \mid h)\) to \(q(x_1 \mid h)\). For time series forecasting tasks, which fundamentally involve learning the mapping from historical sequences to target sequences, a natural and reasonable approach is to incorporate \(h\) as conditional guidance into the velocity field \(u_t\). Consequently, the probability path is also conditioned on \(h\).

Accordingly, our goal is to learn a probability path \( p_{t|H}(x|h) \) and a velocity field \( u_t(x|h) \), ensuring that \( u_t(x|h) \) generates \( p_{t|H}(x|h) \). To achieve this, since the transition path between distributions cannot be directly learned, we perform marginalization over the path.

In time series forecasting tasks, for a given \(h\), there exists a correspondence between the samples of the source distribution and the target distribution. Relying solely on one-sided conditioning, as in previous flow matching methods, is therefore inadequate. To address this, we propose two-sided conditionally guided probability paths, where a marginalization probability path is constructed and integrated to obtain \( p_{t|H}(x|h) \): \vspace{-0.5cm}

\begin{equation} \small
p_{t|H}(x|h) = \int p_{t|0,1,H}(x|x_0,x_1,h) \pi_{0,1|H}(x_0,x_1|h) dx_0 dx_1.
\label{condition_pro}
\end{equation}

Here, \(\pi_{0,1|H}(x_0,x_1 \mid h) = q(x_1 \mid h) p(x_0 \mid h)\), indicating that \(x_0\) and \(x_1\) are independent given \(h\), a concept referred to as conditional independent coupling. Both are related to the historical data \(h\). 

The two-sided conditionally guided probability path is required to comply with the boundary constraints \(p_{0|0,1}(x|x_0, x_1,h) = \delta_{x_0}(x)\) and \(p_{1|0,1}(x|x_0, x_1,h) = \delta_{x_1}(x)\). Here, \(\delta\) denotes the Dirac delta function. Subsequently, the velocity field can be obtained as:\vspace{-0.5cm}

\begin{equation} \small
u_t(x|h) = \int u_t(x|x_0,x_1,h) p_{0,1|t,H}(x_0,x_1|x, h) \, dx_0 dx_1, 
\end{equation}
By Bayes’ Rule, it follows that: \vspace{-0.5cm}

\begin{equation} \small
p_{0,1|t,H}(x_0,x_1|x, h) = \frac{p_{t|0,1,H}(x|x_0,x_1,h) \pi_{0,1|H}(x_0,x_1|h)}{p_{t|H}(x|h)}.
\end{equation}  
Thus, the model learns \( u_t(x | x_0, x_1, h) \) to obtain \( u_t(x | h) \). According to Eq.(\ref{eq:conditional_flow}), \( u_t(x | x_0, x_1, h) \) determines \( p_{0,1|t,H}(x_0,x_1|x,h) \), and vice versa. 

\begin{proposition}[Equivalence of CGFM and Learning Residual Distribution via Flow Matching]  
Under Proposition 0.1 (noise smoothing) and two-sided coupling \(\pi_{0,1|H}=p(x_0|h)q(x_1|h)\), CGFM equivalently learns residual \(\epsilon=x_1-x_0\)'s probabilistic characteristics via Flow Matching, with its path and loss tied to \(\epsilon\)'s evolution.  
\label{Equivalence of CGFM and Error Learning}  
\end{proposition}  

This equivalence highlights a critical advantage of CGFM: instead of treating residuals merely as errors to be minimized through loss functions, it explicitly models the probabilistic structure of these residuals via flow matching. By tying the learning process to the evolution of \(\epsilon = x_1 - x_0\), CGFM leverages the auxiliary model's predictions \(x_0\)—which already encode partial temporal information—as a structured starting point, enabling more efficient learning of how to correct these predictions toward the true future values \(x_1\). 



\subsection{Velocity Field of Marginal Probability Paths}
Previous studies have primarily focused on the application of conditional optimal transport flows. In the scenario of a two-sided condition, this can be formulated as:\(
X_t \sim p_{t|0,1,H} = tX_1 + (1-t)X_0 .
\)
Conditional optimal transport flows addresses the problem of kinetic energy minimization through the optimization formulation:  
\(
\arg\min_{p_t, u_t} \int_0^1 \int_Z \|u_t(x)\|^2 p_t(x) \, dx \, dt,
\)
providing a principled approach to solving such problems. However, in time-series forecasting, this may not identify the optimal predictive path, particularly when the initial distribution is highly complex. Notably, conditional optimal transport can be regarded as a special case within the broader family of affine conditional flows~\cite{albergo2023buildingnormalizingflowsstochastic}. 
\begin{equation}
 X_t \sim p_{t|0,1,H} = \alpha_t X_1 + \beta_t X_0, \label{affine_flow}
\end{equation}
where \(\alpha_t\) and \(\beta_t\) : 
\( [0, 1] \rightarrow [0, 1] \) are smooth functions, satisfying \(\alpha_0 = \beta_1 = 0\) and \(\alpha_1 = \beta_0 = 1\), with \(\dot\alpha_t > 0\), and \(-\dot\beta'_t > 0\) for \(t \in (0, 1)\). 
Referring to Eq.(\ref{eq:conditional_flow}), let \(x' = \psi_t (x)\), and the inverse function yields \(\psi_t^{-1}(x') = x\). Consequently, Eq.(\ref{eq:conditional_flow}) can be reformulated as 
\begin{equation}
u_t(x') = \dot\psi_t(\psi_t^{-1}(x')).
\end{equation}

\begin{proposition}[Velocity Field of Marginal Probability Paths]
Under mild assumptions, if \(\psi_t(\cdot|x_0,x_1,h)\) is smooth in \(t\) and forms a diffeomorphism in \(x_0,x_1\), then the velocity field \(u_t(x)\) can be represented as 
\vspace{-1ex}
\begin{equation}
u_t(x|h) = \mathbb{E}\left[ \dot{\psi}_t(X_0, X_1|H) | X_t = x,H=h \right]
\end{equation}for all \(t \in [0, 1)\).
\label{Velocity of the Conditional Path}
\end{proposition}

According to Proposition \ref{Velocity of the Conditional Path}, under the two-sided condition, the velocity field takes the form of:
\vspace{-1ex}
\begin{equation}
u_t(x|h) = \mathbb{E} \left[ \dot{\alpha}_t X_1 + \dot{\beta}_t X_0 \mid X_t = x, H = h \right].
\label{vel}
\end{equation}
The choice of \(\alpha_t\) and \(\beta_t\) enhances flexibility, making it better suited for complex source distributions and more effective for predictive path construction. We further investigate the effects of different parameterizations of \(\alpha_t\) and \(\beta_t\) in Experiment~\ref{path}.

\subsection{Marginalization}

After deriving the conditional velocity field \( u_t(x|h) \) and the conditional probability path \( p_{t|H}(x|h) \) from the marginal velocity field, we must also ensure that \( u_t(x|h) \) indeed generates \( p_{t|H}(x|h) \).

\begin{proposition}[Marginalization via Conditional Affine
Flows]
Under mild assumptions, $q$ has a bounded support and $p$ is $C^1(\mathbb{R}^d)$ with a strictly positive density and finite second moments. These two are related by the conditional independent coupling $\pi_{0,1|H}(x_0, x_1|h) = p(x_0|h)q(x_1|h)$. $p_t(x|h)$ is defined as Eq.(\ref{condition_pro}), with $\psi_t$ defined by Eq.(\ref{affine_flow}). Subsequently, the marginal velocity engenders $p_t$ that interpolates between $p$ and $q$.
\label{marginalization}
\end{proposition}
According to Proposition \ref{noise_smooth}, whether \( p \) represents the noise distribution or the output distribution of the auxiliary model, appropriate operations can ensure that the \( C^1 \) condition is satisfied. Consequently, Proposition \ref{marginalization} guarantees the correctness of the flow matching construction.

\begin{proposition}[Transportation and Non-crossing of Probability Paths]
Under the assumptions of Proposition~\ref{marginalization}, further suppose that the affine conditional path and the velocity field are given by Eq.~(\ref{affine_flow}) and Eq.~(\ref{vel}), respectively. Then All paths of the flow are non-crossing: there exist no \(t \in [0,1)\), \(z \in \mathbb{R}^d\), and distinct initial conditions \(X_0^{(1)} \neq X_0^{(2)}\) such that \(X_t^{(1)} = X_t^{(2)} = z\) with distinct evolution directions.
\label{Transportation and Non-crossing of Probability Flows}
\end{proposition}
\begin{figure}[!ht]
    \centering
    \includegraphics[width=0.8\linewidth]{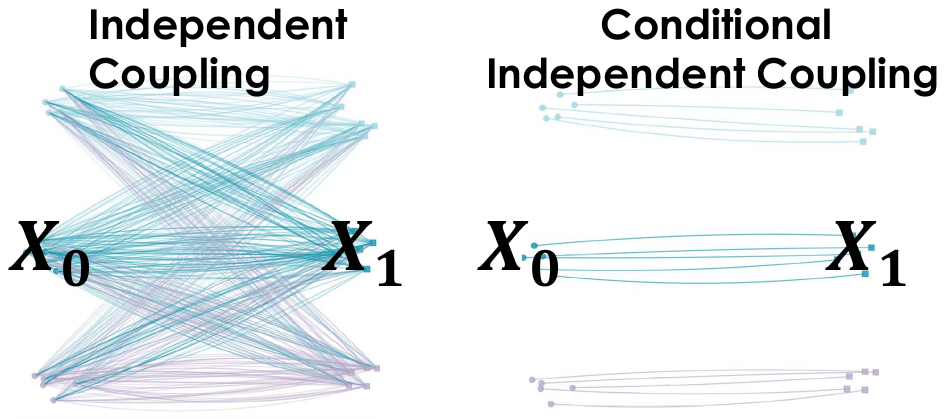}
    \caption{Illustration of Proposition~\ref{Transportation and Non-crossing of Probability Flows}. Left: Independent coupling may produce crossing paths. Right: CGFM’s conditional independent coupling ensures non-crossing paths by conditioning both source and target on shared history}
    \label{fig3}
\end{figure}
As visualized in Figure~\ref{fig3}, this guarantee of non-crossing paths is critical for time series forecasting, as path crossing would lead to ambiguous mappings between the auxiliary model's predictions \(x_0\) and the true future values \(x_1\) at intermediate time steps \(t\). Such ambiguity could cause information loss or conflicting signals during the flow's evolution, undermining the model's ability to refine predictions consistently.


\subsection{Parameterization of the Prediction Target}

Numerous studies in the domains of time-series forecasting and protein synthesis\cite{watson2023novo,shen2024multi,shen2023non}, have undertaken the reparameterization of prediction targets.

In the case of our two-sided condition's affine path By using Eq.(\ref{affine_flow}), we obtain:

\begin{equation}
X_1 = \frac{X_t - \beta_tX_0}{\alpha_t}, X_0 = \frac{X_t - \alpha_tX_1}{\beta_t}.
\label{reparameterization}
\end{equation} 
\vspace{-1ex}

Substituting Eq.(\ref{reparameterization}) into Eq.(\ref{vel}) allows for the following reparameterization:

\begin{equation}
u_t(x|h) = \dot{\alpha}_tE[X_1|X_t = x,h] + \dot{\beta}_tE[X_0|X_t = x,h]    
\label{v1}
\end{equation}
\begin{equation}
= \frac{\dot{\beta}_t}{\beta_t}x + \Big[\dot{\alpha}_t - \frac{\alpha_t \dot{\beta}_t}{\beta_t}\Big]E[X_1|X_t = x,H=h] 
\label{x1pre}
\end{equation}
\begin{equation}
= \frac{\dot{\alpha}_t}{\alpha_t}x + \Big[\dot{\beta}_t - \frac{\beta_t \dot{\alpha}_t}{\alpha_t}\Big]E[X_0|X_t = x,H=h].
\label{x0pre}
\end{equation}
\vspace{-0.5ex}
Whereas Eq.(\ref{x1pre}) provides a parameterization of \( u_t \) for predicting the target \( x_1 \), where \( x_{1|t}(x) = \mathbb{E}[X_1 | X_t = x] \) is defined as the \( x_1 \)-prediction. Eq.(\ref{x0pre}) offers a parameterization of \( u_t \) for the source \( x_0 \), where \( x_{0|t}(x) = \mathbb{E}[X_0 | X_t = x] \) is defined as the \( x_0 \)-prediction. These equations introduce two novel methods of parameterization. In light of pertinent literature \cite{watson2023novo} and experimental observations, we discern that for sequence prediction tasks, considering our generation target $x_1$ as our training objective engenders enhanced outcomes.

\subsection{Loss Function}
After obtaining the conditional guided probability path \( p_{t|H}(x | h) \) and the velocity field \( u_t(x | h) \), we proceed to define the loss function, specifically the guided flow matching loss:
\vspace{-1pt}
\begin{equation}
\mathcal{L}_{GM}(\theta) = \mathbb{E}_{t,H,X_t} \left[ \left\| u_t(X_t \mid H) - u^\theta_t(X_t \mid H) \right\|^2 \right],
\end{equation}
\vspace{-1pt}
where the velocity field \( u_t(X_t \mid H) \) is defined by:\(u_t(X_t \mid H) = \mathbb{E}\left[ g_t(X_0, X_1) \;\bigg|\; X_t = x, H = h \right].\)

The prediction function \( g_t(X_0, X_1) \) is specified based on the prediction target as follows:

\begin{equation}
g_t(X_0, X_1) =
\begin{cases}
    \dot{\alpha}_t X_1 + \dot{\beta}_t X_0, & \text{\( u_t \)-Prediction}, \\
    X_0, & \text{\( X_0 \)-Prediction}, \\
    X_1, & \text{\( X_1 \)-Prediction}.
\end{cases}
\end{equation}

From Eq.(\ref{v1}), Eq.(\ref{x1pre}), and Eq.(\ref{x0pre}), it can be shown that the above three prediction methods are mathematically equivalent. If the prediction objective is \( X_1 \)-prediction, then after training the velocity field \( u_t^\theta \) to predict \( X_1 \), it can be substituted into Eq. \ref{x1pre} to replace \( \mathbb{E}[X_1 \mid X_t = x, H = h] \), resulting in the velocity field at time \( t \). The same principle applies to \( x_0 \). Furthermore, the conditional guided flow matching loss function \(\mathcal{L}_{CGM}(\theta)\) is defined as: 

\begin{equation}
\small
\mathcal{L}_{CGM}(\theta) = \mathbb{E}_{t,H,(X_0,X_1) \sim \pi_{0,1|H}} \left[ \left\| g_t(X_0, X_1) - u^\theta_t(X_t) \right\|^2 \right].
\end{equation} 

Since \( g_t(X_0, X_1) \) is explicitly specified and computable, the loss \(\mathcal{L}_{CGM}(\theta)\) offers significant advantages for optimization.

\begin{proposition}
\label{prop:gradient_equivalence}
The gradients of the guided flow matching loss and the conditional guided flow matching loss coincide:
\begin{equation}
\nabla_\theta \mathcal{L}_{GM}(\theta) = \nabla_\theta \mathcal{L}_{CGM}(\theta). 
\end{equation}
Moreover, the minimizer of the Conditional Guided Flow Matching loss \(\mathcal{L}_{CGM}(\theta)\) is the marginal velocity \( u_t(X_t\mid H) \).
\end{proposition}



\vspace{-1cm}
\section{Experiment}
A detailed ablation study, along with detailed explanations of Figure~\ref{fig:pca-visualization} and Figure~\ref{learning residual}, as well as experimental details such as hyperparameters and additional experiments, can be found in the Appendix.
\subsection{Experiment Settings}
\subsubsection{\textbf{Baseline and Datasets}}
To demonstrate the superiority of our proposed model, CGFM, we compared it against several representative baselines. Among transformer-based models, we included multipatchformer~\cite{naghashi2025multiscale}, iTransformer~\cite{liu2024itransformer}, PatchTST~\cite{nie2023timeseriesworth64}, pathformer~\cite{chen2024pathformer}, Autoformer~\cite{chen2021autoformer}, and the classic FedFormer~\cite{zhou2022fedformer}. Additionally, we incorporated MLP-based models, including RLinear~\cite{li2023revisitinglongtermtimeseries}, TimesNet~\cite{wu2022timesnet}, Timemixer~\cite{wang2024timemixer} and TiDE~\cite{das2024longtermforecastingtidetimeseries}. Furthermore, diffusion-based models such as CSDI~\cite{tashiro2021csdi} and TimeDiff~\cite{shen2023non} were also evaluated. Descriptions of the datasets and their statistical properties are detailed in the Appendix.
\subsubsection{\textbf{Evaluation Metrics}} The experiments employed Mean Squared Error (MSE) and Mean Absolute Error (MAE) to evaluate the predictive performance of the models. To ensure the robustness of the results, each experiment was repeated 10 times, and the outcomes were averaged.
\begin{table*}[!ht]

  \scriptsize
  \small
  \renewcommand\arraystretch{0.37}
  \tabcolsep=0.08cm
  \centering
  \caption{Forecasting errors under the multivariate setting. The \textbf{bold} values indicate better performance.}
    \resizebox{\textwidth}{!}{
        \begin{tabular}{c|c|cccc|cccc|cccc|cccc}
    \toprule
    \multicolumn{2}{c|}{Methods} & \multicolumn{2}{c}{Rlinear} & \multicolumn{2}{c}{+ CGFM} & \multicolumn{2}{c}{iTransformer} & \multicolumn{2}{c}{+ CGFM} & \multicolumn{2}{c}{TimeDiff} & \multicolumn{2}{c}{+ CGFM} & \multicolumn{2}{c}{MultiPatch.} & \multicolumn{2}{c}{+ CGFM} \\
    \multicolumn{2}{c|}{Metric} & MSE   & MAE   & MSE   & \multicolumn{1}{c}{MAE} & MSE   & MAE   & MSE   & \multicolumn{1}{c}{MAE} & MSE   & MAE   & MSE   & \multicolumn{1}{c}{MAE} & MSE   & MAE   & MSE   & MAE \\
    \midrule

    \multirow{4}[2]{*}{\rotatebox{90}{ETTm1}}
           & 96    & 0.359  & 0.378  & \textbf{0.307} & \textbf{0.351} & 0.336  & 0.369  & \textbf{0.313} & \textbf{0.362} & 0.339  & 0.362  & \textbf{0.309} & \textbf{0.361} & 0.317  & 0.345  & \textbf{0.308} & \textbf{0.355} \\
          & 192   & 0.396  & 0.395  & \textbf{0.341} & \textbf{0.382} & 0.387  & 0.392  & \textbf{0.366} & \textbf{0.382} & 0.372  & 0.381  & \textbf{0.346} & \textbf{0.389} & 0.367  & 0.369  & \textbf{0.339} & \textbf{0.378} \\
          & 336   & 0.428  & 0.416  & \textbf{0.372} & \textbf{0.397} & 0.427  & 0.422  & \textbf{0.398} & \textbf{0.412} & 0.403  & 0.401  & \textbf{0.384} & \textbf{0.409} & 0.399 & 0.398 & \textbf{0.373} & \textbf{0.401} \\
          & 720   & 0.489  & 0.451  & \textbf{0.443} & \textbf{0.421} & 0.493  & 0.461  & \textbf{0.461}  & \textbf{0.452} & 0.455  & 0.432  & \textbf{0.441} & \textbf{0.416} & 0.467 & 0.436  & \textbf{0.438} & \textbf{0.430} \\
            \midrule

        \multirow{4}[2]{*}{\rotatebox{90}{ETTm2}} 
          & 96    & 0.182 & 0.267  & \textbf{0.167} & \textbf{0.253} &\textbf{0.179}  & \textbf{0.262}  & 0.181  & 0.265  & 0.185  & 0.265  & \textbf{0.170} & \textbf{0.261} & 0.171  & 0.252& \textbf{0.165} & \textbf{0.255} \\
          & 192   & 0.246  & 0.305  & \textbf{0.228} & \textbf{0.298} & 0.244  & 0.306  & \textbf{0.242} & \textbf{0.299} & 0.251  & 0.310  & \textbf{0.234} & \textbf{0.286} & 0.238  & 0.296  & \textbf{0.229} & \textbf{0.296} \\
          & 336   & 0.310  & 0.344  & \textbf{0.281} & \textbf{0.323} & 0.314  & 0.351  & \textbf{0.291} & \textbf{0.329} & 0.311  & 0.352  & \textbf{0.283} & \textbf{0.315} & 0.305  & 0.342  & \textbf{0.283} & \textbf{0.320} \\
          & 720   & 0.407  & 0.399  & \textbf{0.365} & \textbf{0.367} & 0.413  & 0.407  & \textbf{0.380}  & \textbf{0.391} & 0.412  & 0.399  & \textbf{0.373} & \textbf{0.386} & 0.404  & 0.403  & \textbf{0.366} & \textbf{0.371} \\
            \midrule
        \multirow{4}[2]{*}{\rotatebox{90}{ETTh1}} 
          & 96    & 0.382  & 0.398  & \textbf{0.363} & \textbf{0.372} & 0.389 & 0.408  & \textbf{0.368} & \textbf{0.388} & 0.383  & 0.391  & \textbf{0.371} & \textbf{0.386} & 0.378  & 0.389  & \textbf{0.365} & \textbf{0.371} \\
          & 192   & 0.439 & 0.424  &\textbf{0.409} & \textbf{0.417} & 0.443  & 0.441  & \textbf{0.410} & \textbf{0.423} & 0.437  & 0.429  & \textbf{0.415} & \textbf{0.421} & 0.434  & 0.422  & \textbf{0.422} & \textbf{0.435} \\
          & 336   & 0.480  & 0.448  & \textbf{0.425} & \textbf{0.430} & 0.489  & 0.461  & \textbf{0.428} & \textbf{0.437} & 0.475  & 0.449  & \textbf{0.423} & \textbf{0.432} & 0.473  & 0.445  & \textbf{0.439} & \textbf{0.48} \\
          & 720   & 0.484  & 0.475  & \textbf{0.461} & \textbf{0.457} & \textbf{0.506}  & \textbf{0.498}  & 0.511  & 0.503 & 0.502  & 0.512  & \textbf{0.476} & \textbf{0.495} & 0.476  & 0.470 & \textbf{0.462} & \textbf{0.461} \\
    \midrule

    \multirow{4}[2]{*}{\rotatebox{90}{ETTh2}}
          & 96    & 0.290  & 0.341  & \textbf{0.275} & \textbf{0.329} & \textbf{0.299}  & \textbf{0.351}  & 0.300 & 0.357 & 0.301  & 0.357  & \textbf{0.282} & \textbf{0.346} & 0.285  & 0.334  & \textbf{0.280}  & \textbf{0.328} \\
          & 192   & 0.375  & 0.392  & \textbf{0.351} & \textbf{0.362} & 0.383  & 0.402  & \textbf{0.377} & \textbf{0.398} & 0.381  & 0.396  & \textbf{0.372} & \textbf{0.391} & 0.371  & 0.389  & \textbf{0.365} & \textbf{0.372} \\
          & 336   & 0.414  & 0.426  & \textbf{0.402} & \textbf{0.422} & 0.431  & 0.435  & \textbf{0.423} & \textbf{0.431} & 0.433 & 0.441  & \textbf{0.425} & \textbf{0.432} & 0.420  & 0.428  & \textbf{0.415} & \textbf{0.417} \\
          & 720   & 0.422  & 0.447  & \textbf{0.411} & \textbf{0.442} & 0.429 & 0.448  & \textbf{0.423}  &  \textbf{0.445} & 0.437  & 0.458  & \textbf{0.419} & \textbf{0.445} & 0.425 & 0.441 & \textbf{0.419} & \textbf{0.432} \\
            \midrule
    \multirow{4}[2]{*}{\rotatebox{90}{Ex.}} 
          & 96    & 0.095  & 0.215  & \textbf{0.081} & \textbf{0.204}  & 0.089  & 0.218  & \textbf{0.082} & \textbf{0.206} & 0.087  & 0.212  & \textbf{0.082} & \textbf{0.203} & 0.085 & 0.206  & \textbf{0.080} & \textbf{0.201} \\
          & 192   & 0.182  & 0.308  & \textbf{0.175}  & \textbf{0.304}  & 0.177  & 0.301  & \textbf{0.174} & \textbf{0.308} & 0.176  & 0.311  & \textbf{0.174} & \textbf{0.307} & 0.178  & 0.297  & \textbf{0.173} & \textbf{0.302} \\
          & 336   & 0.349  & 0.432  & \textbf{0.306} & \textbf{0.395} & 0.336  & 0.421  & \textbf{0.306} & \textbf{0.397} & 0.310  & 0.427  & \textbf{0.305} & \textbf{0.399} & 0.307 & 0.399  & \textbf{0.304} & \textbf{0.394} \\
          & 720   & 0.890  & 0.719  & \textbf{0.830} & \textbf{0.683} & 0.851  & 0.693  & \textbf{0.837} & \textbf{0.690} & 0.847 & 0.706  & \textbf{0.844}  & \textbf{0.701}  & 0.897  & 0.702  & \textbf{0.844} & \textbf{0.685} \\
    \midrule
    \multirow{4}[2]{*}{\rotatebox{90}{Traffic}} 
          & 96    &0.632& 0.387  & \textbf{0.412}  & \textbf{0.288}  & 0.397  & 0.272  & \textbf{0.388} & \textbf{0.268} & 0.520  & 0.373  & \textbf{0.398} & \textbf{0.277} & 0.442  & 0.268  & \textbf{0.423} & \textbf{0.253} \\
          & 192   & 0.597 & 0.362  & \textbf{0.429}  & \textbf{0.291}  & 0.422  & 0.278  & \textbf{0.413} & \textbf{0.269} & 0.515  & 0.354  & \textbf{0.427} & \textbf{0.281} & 0.460  & 0.273  & \textbf{0.441} & \textbf{0.266} \\
          & 336   & 0.607  & 0.369  & \textbf{0.462}  & \textbf{0.336}  & 0.437  & 0.288  & \textbf{0.428} & \textbf{0.276} & 0.514  & 0.355  & \textbf{0.459} & \textbf{0.322} & 0.477  & 0.276  & \textbf{0.454} & \textbf{0.271} \\
          & 720  & 0.650  &0.391  & \textbf{0.489}  & \textbf{0.323}  & 0.473  & 0.304  & \textbf{0.462} & \textbf{0.296} & 0.563  & 0.377  & \textbf{0.478} & \textbf{0.310} & 0.517  & 0.299  & \textbf{0.471} & \textbf{0.286 } \\
    \midrule
    \multirow{4}[2]{*}{\rotatebox{90}{Weather}} 
          & 96    & 0.189  & 0.230  & \textbf{0.152} & \textbf{0.191} & 0.178  & 0.217  & \textbf{0.154} & \textbf{0.193} & 0.181  & 0.217  & \textbf{0.156} & \textbf{0.191} & 0.157  & 0.197  & \textbf{0.151} & \textbf{0.186} \\
          & 192   & 0.244  & 0.275  & \textbf{0.201} & \textbf{0.226} & 0.224  & 0.259  & \textbf{0.204} & \textbf{0.241} & 0.228  & 0.257  & \textbf{0.202} & \textbf{0.239} & 0.207  & 0.242  & \textbf{0.195} & \textbf{0.221} \\
          & 336   & 0.295  & 0.309  & \textbf{0.259} & \textbf{0.272} & 0.281 & 0.298  & \textbf{0.270} & \textbf{0.289} & 0.288  & 0.303  & \textbf{0.273} & \textbf{0.292} & 0.277  & 0.293 & \textbf{0.243} & \textbf{0.254} \\
          & 720   & 0.368  & 0.355  & \textbf{0.339} & \textbf{0.332} & 0.359  & 0.350  & \textbf{0.343} & \textbf{0.334} & 0.364  & 0.358  &\textbf{ 0.347} & \textbf{0.340}  & 0.351  & 0.342  & \textbf{0.332} & \textbf{0.321} \\

    \bottomrule
    \end{tabular}%
    }
\label{table1}

\end{table*}%

\begin{algorithm}[H]
\caption{CGFM Training}\label{alg:training}
\begin{algorithmic}[1]  
\State \textbf{Input:} History distribution \( p_H \), path parameters \( \alpha_t, \beta_t \), smoothing level \( \sigma \), network \( u^\theta_t \), source distribution \( p(x_0 \mid h) \), source mode: noise or auxiliary output, prediction objective \( g_t \), and target distribution \( q(x_1 \mid h) \).
\Repeat
    \State $h \sim p_H$
    \State $x_0 \sim p(x_0|h);\ x_1 \sim q(x_1|h)$
    \If {source mode == auxiliary output}  
        \State $\varepsilon \sim \mathcal{N}(0, I)$  
        \State $x_0 \gets x_0 + \sigma \varepsilon$  
    \EndIf
    \State $t \sim \mathcal{U}(0, 1)$
    \State $x_t \gets \alpha_t x_1 + \beta_t x_0$ 
    \State $\mathcal{L}_{CGM}(\theta) \gets \left\| g_t(x_0, x_1) - u^\theta_t(x_t|h) \right\|^2 $
    \State $\theta \gets \mathrm{Update}(\theta, \nabla_{\theta} \mathcal{L}_{CGM}(\theta))$
\Until {converged}  
\State \textbf{Return:} $u_t^\theta$ 
\end{algorithmic}
\end{algorithm}

\subsubsection{Auxiliary Model Enhanced Performance}

As shown in Table~\ref{table1}, for an input length of 96, the auxiliary model is evaluated across a wide range of mainstream forecasting models, including MLP-based, Transformer-based, and diffusion-based architectures. Our proposed CGFM framework achieves significant performance improvements across most benchmark datasets. 

Specifically, CGFM achieves the greatest improvement when using Rlinear as the auxiliary model, while the gains for iTransformer are comparatively smaller. Notably, iTransformer occasionally exhibits performance degradation; this phenomenon is examined in more detail in Figure~\ref{fig:pca-visualization}. As shown on the left of Figure~\ref{fig:pca-visualization}, the PCA trajectory of RLinear predictions closely aligns with the ground truth in both shape and continuity, indicating that RLinear effectively captures the temporal evolution of the underlying physical process. In contrast, the middle of Figure~\ref{fig:pca-visualization} shows that iTransformer predictions exhibit an irregular spatial distribution, with discontinuities in inter-point connections. This suggests non-smooth fluctuations in the high-dimensional representation of iTransformer’s predictions. These observations suggest that Transformer architectures are more prone to disrupting the smoothness of the data, thereby posing challenges for learning effective flow paths.

\begin{figure}[t]
\centering
\includegraphics[width=0.98\linewidth]{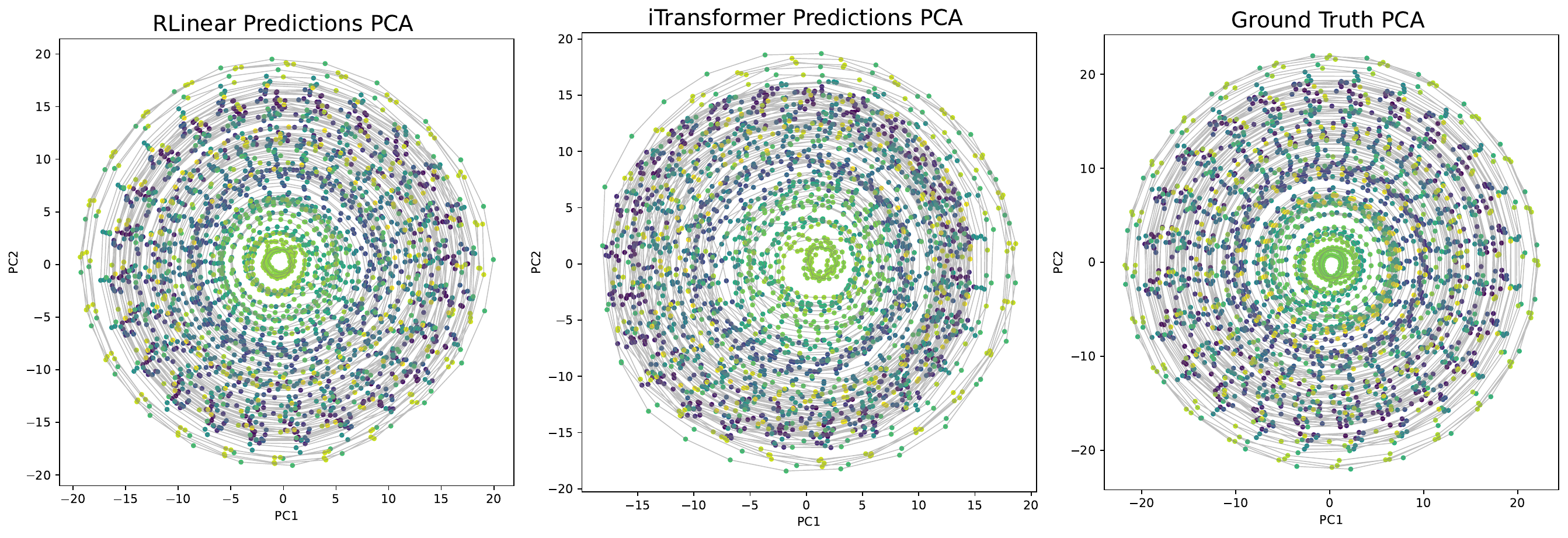}
\caption{PCA visualization of predictions and ground truth, showing the PCA projections of RLinear predictions, iTransformer predictions, and the ground truth, respectively.}
\vspace{-2ex}
\label{fig:pca-visualization}
\end{figure}

\subsubsection{Performance without Auxiliary Model}
\begin{table*}[h]

\small
\centering
\caption{Testing MSE in the multivariate setting. Number in brackets is the rank. CSDI runs out of memory on \textit{Traffic}.}
\def\arraystretch{0.5}
\resizebox{\textwidth}{!}{
\begin{tabular}{lccccccc|c}
\toprule
\textbf{Model} & \textbf{Weather}   & \textbf{Traffic}  & \textbf{ETTh1} & \textbf{ETTh2} & \textbf{ETTm1} & \textbf{ETTm2} & \textbf{Exchange} & \textbf{Avg Rank} \\
\midrule

\textbf{CGFM } & 0.161(3) & 0.430(2)     & 0.373(1) & 0.286(2)  & 0.317(2) & 0.173(2)  & 0.085(3) & \textbf{2.142 (1)}\\

\textbf{TimeDiff} & 0.181(8)  & 0.520(7)    & 0.383(7) & 0.301(7) &  0.339(8)  & 0.185(8) & 0.087(6) & 7.125(7)\\

\textbf{CSDI} & 0.301(13)   & -             & 0.503(13) & 0.356(11)  & 0.601(13) & 0.289(13)  & 0.258(13) & 12.667(13)\\

\textbf{iTransformer} & 0.178(6)  & 0.397(1) & 0.389(10) & 0.299(6)  & 0.336(6)  & 0.179(6) & 0.089(7)   & 6.000(6)\\

\textbf{Rlinear} & 0.189(9) & 0.632(11) & 0.382(5) & 0.290(4) & 0.359(9) & 0.182(7) & 0.095(8) & 7.429(8)\\

\textbf{FedFormer} & 0.219(11)    & 0.588(8)  &  0.376(3)  & 0.359(12)  & 0.379(11) & 0.203(10)  & 0.147(11)   & 7.571(9)\\

\textbf{TimeMixer}& 0.165(4) & 0.461(4) &   0.374(2)  & 0.294(5)  & 0.331(5) & 0.175(4)  &  0.083(1)  & 3.571(4)\\

\textbf{TimesNet} & 0.179(7)  & 0.593(9) & 0.384(9)  & 0.340(9) & 0.338(7)  & 0.188(9) & 0.107(10)  & 8.571(10)\\

\textbf{PatchTST} & 0.177(5) & 0.462(5)   & 0.383(7) & 0.304(8) & 0.320(4) & 0.175(4)  & 0.085(3)  & 5.142(5)\\
\textbf{Autoformer} & 0.266(12)  & 0.613(10)& 0.449(11) & 0.345(10)  & 0.505(12) & 0.255(12) & 0.189(12) & 11.286(12)\\
\textbf{TiDE} & 0.202(10)  & 0.803(12)  & 0.478(12) & 0.403(13) &  0.366(10) & 0.209(11)  & 0.093(8) & 10.857(11)\\
\textbf{Pathformer}& 0.156(1)   & 0.479(6) &  0.382(5) &  0.283(1)  & 0.319(3) & 0.174(3)  &   0.083(1)  &2.857(3)\\
\textbf{MultiPatchFormer}& 0.158(2)   & 0.438(2) &  0.378(4) &  0.285(3)  & 0.315(1) & 0.171(1)  &   0.085(3)  &2.285(2)\\
\bottomrule
\vspace{-0.3in}
\label{noisex_0}
\end{tabular}}
\end{table*}

To further validate the superiority of our proposed forecasting framework, we conducted an additional experiment where the initial condition $x_0$ was directly set as noise under the 96-to-96 prediction setting. CGFM achieved the best overall ranking, demonstrating the effectiveness and soundness of its model architecture. Although it does not attain the best performance on every individual dataset, as shown in Table~\ref{table1}, when combined with an auxiliary model, CGFM consistently achieves the best results across all datasets. This highlights the crucial role of the CGFM auxiliary model in enhancing forecasting performance.

\subsubsection{Case Study of residual learning}

As shown in Figure~\ref{learning residual}, we evaluate the 96-to-96 prediction task on the ETTh1 dataset using RLinear as the auxiliary model. The curves represent the mean sequences obtained by averaging all predicted and ground truth windows of length 96. It can be observed that RLinear exhibits persistent underestimation and large variance. In contrast, CGFM effectively learns the residual distribution, substantially reducing both variance and systematic bias, resulting in predictions that closely align with the ground truth.
\vspace{-0.5cm}
\begin{figure}[!ht]
    \centering
    \includegraphics[width=1\linewidth]{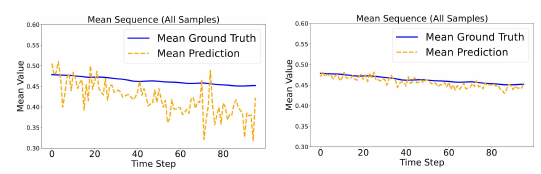}
   \caption{Comparison of results: Left (before CGFM) and Right (after CGFM).}

    \label{learning residual}
\end{figure}

\subsubsection{Analysis of Path Hyperparameter}
\label{path}

We explore different parameterization schemes for affine probabilistic paths, including Optimal Transport (CondOT), Polynomial (Poly-n), Linear Variance Preserving (LinearVP), and Cosine schedulers. Detailed formulas are provided in Appendix. 

Table~\ref{mse_mae_results} shows that the Poly-n scheduler, with velocity approaching zero near \(t \approx 0\), enables thorough exploration around \(X_0\), similar to the denoising phase in diffusion models. This extension effectively increases the model's decision time, enhancing its ability to capture intricate details early on. 


\subsubsection{Analysis of Prediction Functions}

To study different prediction functions' impact on time-series forecasting, we tested ETTh1 and ETTm1 datasets with RLinear as the auxiliary model. Though mathematically equivalent, results (Table~\ref{tab:prediction_func}) show \(X_1\)-Prediction consistently achieves the lowest MSE and MAE, outperforming \(X_0\)- and \(u_t\)-Prediction. Intuitively, \(X_1\)-Prediction directly targets the future series, \(X_0\)-Prediction focuses on noise, and \(u_t\)-Prediction mixes both. This aligns with prior work~\cite{watson2023novo,shen2023non}.


\begin{table}[t]
    \centering
    \footnotesize
    \caption{MSE and MAE results for different affine conditional paths on ETTh1 and ETTm1 datasets.}
    \renewcommand\arraystretch{0.95} 
    \begin{adjustbox}{max width=0.98\linewidth} 
        \begin{tabular}{c|cc|cc|cc|cc}
            \toprule
            Dataset & \multicolumn{2}{c|}{CondOT} & \multicolumn{2}{c|}{Poly-n} & \multicolumn{2}{c|}{VP} & \multicolumn{2}{c}{Cosine} \\
            & MSE & MAE & MSE & MAE & MSE & MAE & MSE & MAE \\
            \midrule
            ETTh1 (Rlinear) & 0.368 & 0.376 & \textbf{0.363} & \textbf{0.372} & 0.379 & 0.387 & 0.380 & 0.396 \\
            ETTm1 (Rlinear) & 0.314 & 0.363 & \textbf{0.307} & \textbf{0.351} & 0.336 & 0.376 & 0.332 & 0.371 \\
            ETTh1 (iTrans) & 0.374 & 0.391 & \textbf{0.368} & \textbf{0.388} & 0.387 & 0.403 & 0.387 & 0.386 \\
            ETTm1 (iTrans) & 0.326 & 0.367 & \textbf{0.313} & \textbf{0.362} & 0.331 & 0.372 & 0.334 & 0.376 \\
            \bottomrule
        \end{tabular}
    \end{adjustbox}
    \label{mse_mae_results}
\end{table}

\begin{table}[t]
    \centering
    \captionsetup{type=table}
    \footnotesize 
    \caption{MSE and MAE results for different prediction functions on ETTh1 and ETTm1 datasets.}
    \begin{adjustbox}{max width=1.0\textwidth}
        \begin{NiceTabular}{l cc cc cc} 
            \toprule
            Dataset & \multicolumn{2}{c}{$u_t$-Prediction} 
                    & \multicolumn{2}{c}{$X_0$-Prediction} 
                    & \multicolumn{2}{c}{$X_1$-Prediction} \\ 
            \cmidrule(lr){2-3} \cmidrule(lr){4-5} \cmidrule(lr){6-7}
                     & MSE & MAE & MSE & MAE & MSE & MAE \\ 
            \midrule
            ETTh1    & 0.370 & 0.379 & 0.384 & 0.385 & \textbf{0.363} & \textbf{0.372} \\ 
            ETTm1    & 0.328 & 0.361 & 0.343 & 0.367 & \textbf{0.307} & \textbf{0.351} \\ 
            \bottomrule
        \end{NiceTabular}
    \end{adjustbox}
    \label{tab:prediction_func}
\end{table}

\vspace{-2ex}
\section{Conclusion}
In this paper, we propose Conditional Guided Flow Matching (CGFM), a novel model-agnostic framework to enhance time series forecasting. By extending flow matching with auxiliary model outputs, CGFM learns from prediction residuals—structured patterns like systematic biases often overlooked when residuals are treated merely as optimization targets. CGFM incorporates historical data as both conditions and guidance, uses two-sided conditional probability paths, and employs general affine paths to expand the probability path space. These designs avoid path crossing without complex mechanisms, preserve temporal consistency, and strengthen distribution alignment. Extensive experiments show CGFM consistently improves performance and outperforms state-of-the-art models across diverse datasets, validating its effectiveness in advancing time series forecasting .

\bibliography{aaai2026}


\end{document}